\documentclass{article}

\usepackage[preprint]{neurips_2026}

\usepackage[utf8]{inputenc}
\usepackage[T1]{fontenc}
\usepackage[hidelinks]{hyperref}
\usepackage{url}
\usepackage{booktabs}
\usepackage{amsfonts}
\usepackage{nicefrac}
\usepackage{microtype}
\usepackage{xcolor}
\usepackage{amsmath,amssymb,amsthm,mathtools}
\usepackage{graphicx}
\usepackage{enumitem}
\usepackage{multirow}
\usepackage{array}
\usepackage{tikz}
\usetikzlibrary{positioning,arrows.meta,calc}

\title{A Controlled Counterexample to Strong Proxy-Based Explanations of OOD Performance:\\
in a Fixed Pretraining-and-Probing Setup}

\author{%
Hongmin Li\textsuperscript{1,2}\\[0.25em]
\normalfont\textsuperscript{1}School of Life Science and Technology, Institute of Science Tokyo\\
\normalfont 2-12-1 Ookayama, Meguro-ku, Tokyo 152-8550, Japan\\
\normalfont\textsuperscript{2}Department of Computational Biology and Medical Sciences\\
\normalfont Graduate School of Frontier Sciences, The University of Tokyo\\
\normalfont 5-1-5 Kashiwanoha, Kashiwa-shi, Chiba 277-8561, Japan\\
\normalfont\texttt{lihongmin@edu.k.u-tokyo.ac.jp}\\[0.25em]
\normalfont\small Researcher, School of Life Science and Technology, Institute of Science Tokyo;\\
\normalfont\small Guest Researcher, Department of Computational Biology and Medical Sciences,\\
\normalfont\small Graduate School of Frontier Sciences, The University of Tokyo.\\
\normalfont\small ORCID: \href{https://orcid.org/0000-0003-0228-0600}{0000-0003-0228-0600}
}

\newtheorem{definition}{Definition}
\newtheorem{proposition}{Proposition}
\newtheorem{remark}{Remark}

\newcommand{\cS}{\mathcal{S}}
\newcommand{\cD}{\mathcal{D}}
\newcommand{\cT}{\mathcal{T}}

\newcommand{\OODPerf}{\mathrm{OODPerf}}

\begin{document}

\maketitle

\begin{abstract}
Task-agnostic structure proxies are often used to interpret why one pretraining
corpus transfers better than another, but such explanations require the proxy to
track the structure that matters for the downstream task. We test this
requirement in a fixed pretraining-and-probing setup motivated by
computationally bounded notions of learned structure, including
\emph{epiplexity}. The core question is whether a proxy ranking of two
pretraining datasets must agree with their ranking by OOD probe accuracy. We
show that it need not. First, we give a controlled construction in which a
formal structure quantity, its operational proxy, and the task-relevant
structure for a target family separate. We then instantiate the same mechanism
in a synthetic sequence-model experiment: under the primary all-sample
evaluation, the OOD accuracy ranking reverses the proxy ranking in two of three
seeds, with auxiliary diagnostics and ablations supporting the same
interpretation. The counterexample does not reject structure-based explanations
in general; it identifies a boundary on strong proxy-based explanations. A
proxy for total learned structure can fail to track the task-relevant structure
that drives OOD performance, even in a controlled setting.

\end{abstract}

\section{Introduction}
Large-scale pretraining has renewed a basic question: why does one pretraining corpus support stronger out-of-distribution (OOD) performance than another? One recent answer is to appeal to computationally bounded notions of learned structure, including \emph{epiplexity}, and to estimate them with task-agnostic proxies \cite{finzi2026entropy}. That line of work does more than report correlations. It also motivates stronger explanations in which a proxy-based ranking over corpora is expected to align with OOD performance under a fixed evaluation setup.

Some proxy-based explanatory claims in that narrative conflate three distinct objects:
\begin{enumerate}[leftmargin=1.2em]
    \item a \textbf{formal structure quantity} defined at the level of a computationally bounded coding problem;
    \item an \textbf{operational proxy} instantiated through a particular model family, optimizer, representation, and hyperparameter search procedure;
    \item \textbf{task-relevant structure} for a target task family, i.e., the component of learned regularity that actually supports stronger task-specific OOD performance.
\end{enumerate}
Even after fixing a single pretraining-and-probing pipeline and its OOD probe-accuracy metric, these three layers need not coincide. The stronger explanatory requirement is simply that, inside that setup, a proxy should not put dataset A ahead of dataset B while dataset B still achieves higher OOD probe accuracy.

The key conceptual distinction is between \emph{total learned structure} and \emph{task-relevant structure}. A pretraining distribution may contain abundant regularity that is easy to compress and learn, yet largely irrelevant to a target task. Another distribution may contain less total regularity, but more of the regularity that actually supports OOD performance under a fixed evaluation. This separation permits a controlled counterexample.

Figure~\ref{fig:counterexample-overview} summarizes the construction: the same fixed learner, representation, and probe can support different rankings depending on whether the evaluation lens rewards total compressible structure or task-relevant structure under shift.

\begin{figure}[t]
\centering
\includegraphics[width=\linewidth]{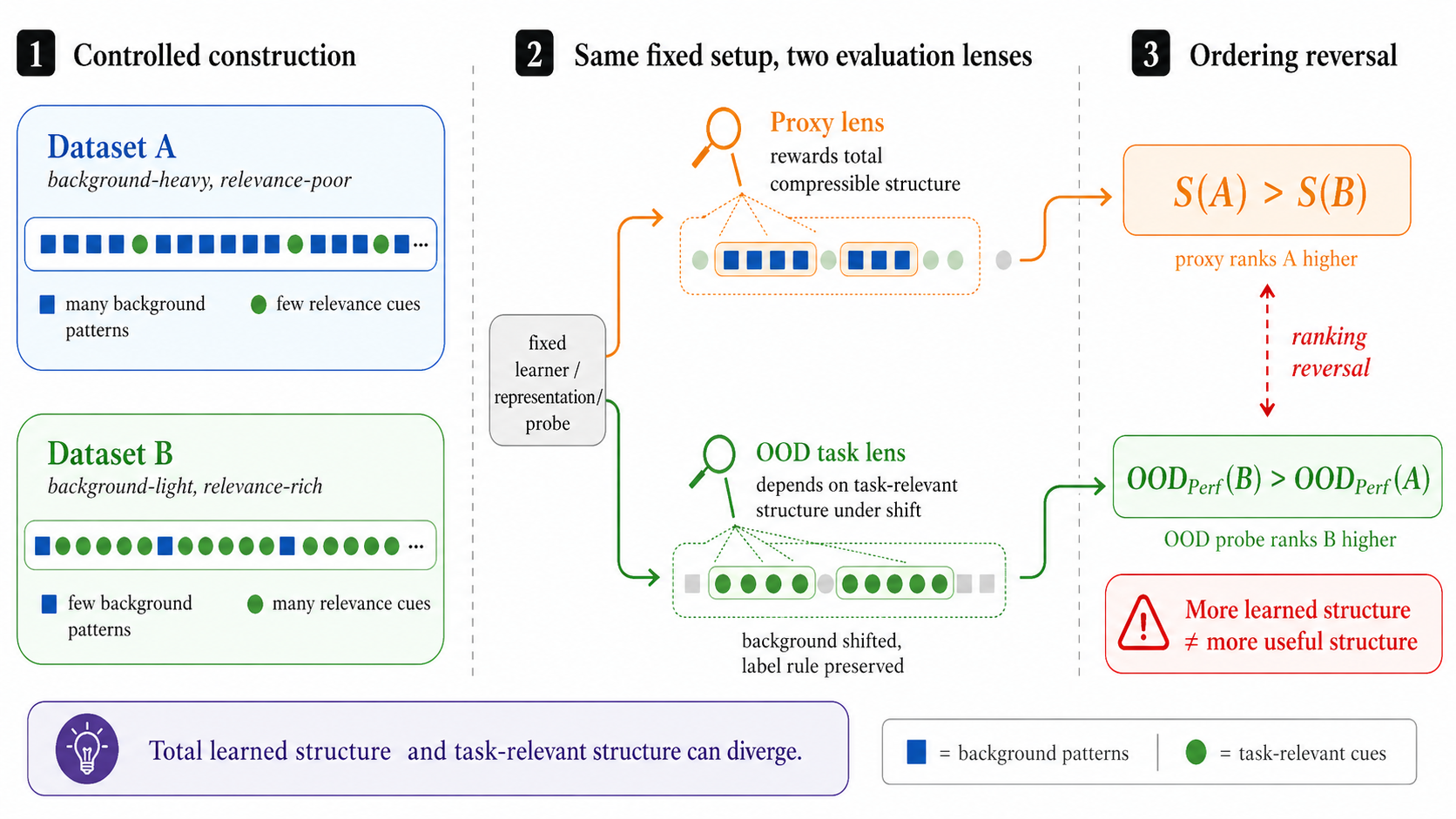}
\caption{Conceptual overview of the controlled counterexample. Dataset A contains more total compressible background structure but fewer task-relevant cues, while dataset B contains less background structure and more task-relevant cues. Under the same fixed pretraining-and-probing setup, a task-agnostic proxy can rank A above B even when the OOD task evaluation ranks B higher.}
\label{fig:counterexample-overview}
\end{figure}

The paper makes the following contributions.
\begin{itemize}[leftmargin=1.2em]
    \item We formalize a three-layer view of structure: \emph{formal}, \emph{operational}, and \emph{task-relevant} structure.
    \item We give an existential controlled counterexample showing that, within a fixed pretraining-and-probing setup, a task-agnostic proxy for total learned structure can rank dataset A above dataset B even when dataset B achieves higher OOD probe accuracy under that setup.
    \item We provide a concrete synthetic data design and a reproducible sequence-model protocol that instantiates the same failure mode empirically.
    \item Empirically, the primary all-sample evaluation reverses the proxy ranking in two of three seeds. A secondary informative-subset diagnostic and complementary ablations support the same mechanism-level interpretation while remaining clearly secondary to the primary reversal.
\end{itemize}

The counterexample therefore shows that, within this fixed pretraining-and-probing setup, strong proxy-based explanations of OOD performance require tracking \emph{task-relevant structure} more closely than total learned structure alone.

\section{Related Work}
\paragraph{Epiplexity and proxy-based explanations.}
The epiplexity program distinguishes unpredictable random content from structural content that a bounded learner can write into its weights \cite{finzi2026entropy}. In that formulation, computation, representation, and observer limitations are part of the quantity being defined.

Empirical work built around that framework necessarily passes through restricted model classes and coding proxies rather than direct optimization of the original formal quantity. A further interpretive step is then required to connect such proxies to task-specific OOD performance on a particular task family. The analysis here focuses on that step: it isolates the distinction between a formal structure quantity, an operational proxy, and task-relevant structure within the fixed pretraining-and-probing setup used in this paper.

\paragraph{Probe-based evaluation and representational claims.}
A nearby literature studies what can and cannot be inferred from probes of learned representations. Linear and shallow nonlinear probes are useful operational tools, but their scores depend jointly on the representation, the probe family, and the decoding budget \cite{alain2016understanding,hewitt2019designing}. MDL-style probing makes the same point in information-theoretic language: an operational score combines information present in the representation with the complexity and inductive bias of the chosen decoder \cite{voita2020information}. The empirical metric used here inherits the same dependence. The empirical claim is therefore about task-specific OOD performance as measured by OOD probe accuracy under this fixed pretraining-and-probing setup, rather than about an observer-independent notion of performance.

\paragraph{Compression-style proxies versus formal coding quantities.}
The empirical proxy studied here is much closer to an MDL-style operational score than to an exact optimization of the formal epiplexity objective. The description-length literature has long separated ideal coding quantities from practical surrogates, restricted model classes, and finite search procedures \cite{rissanen1978modeling,grunwald2007mdl}. The same separation applies here. The relevant question is whether an operational compression-style score can rank dataset A above dataset B while OOD probe accuracy under that same setup is still higher for dataset B, absent additional alignment assumptions.

\paragraph{Task-conditioned data valuation.}
A separate literature scores training examples or datasets by their contribution to a specified supervised objective, for example through Shapley-style data valuation or training-dynamics diagnostics \cite{ghorbani2019data,swayamdipta2020dataset}. These methods are explicitly task-conditional. The present analysis instead considers whether a task-agnostic proxy computed from pretraining alone can rank dataset A above dataset B while later task-specific OOD performance ranks dataset B above dataset A. The point at issue is therefore the explanatory step from an operational proxy to task-specific OOD performance, not the formal quantity in isolation. Task-conditioned data valuation methods incorporate the target objective directly, whereas epiplexity-style proxy claims aim to explain performance differences before the target task is incorporated into the score.

\paragraph{Boundary of proxy-based explanations.}
The controlled construction and empirical instantiation reported here are compatible with settings in which more learned structure correlates with better task-specific OOD performance. The narrower question is whether a total-structure proxy can rank dataset A above dataset B while OOD probe accuracy under the same fixed pretraining-and-probing setup is still higher for dataset B. The contribution is a controlled counterexample to the stronger interpretation that such rankings should stay aligned, while remaining compatible with positive correlations in settings with additional alignment.

\section{Method}
The analysis distinguishes three quantities that are often discussed together.

\begin{definition}[Formal structure]
A \emph{formal structure quantity} is a quantity defined by an explicit coding or optimization problem over an abstract class of bounded observers or programs. In the epiplexity setting, the target formal object is the minimal model description length within a time-bounded two-part code \cite{finzi2026entropy}.
\end{definition}

\begin{definition}[Operational structure]
An \emph{operational structure proxy} is any computable statistic intended to approximate or stand in for a formal structure quantity, but instantiated through a specific model family, representation, optimizer, training protocol, and finite hyperparameter search. Operational structure is therefore pipeline-relative in an engineering sense.
\end{definition}

\begin{definition}[Task-relevant structure]
Given a target task family $\cT$, the \emph{task-relevant structure} of a pretraining distribution $\cD$ is the component of its learnable regularity that supports successful task-specific OOD performance on tasks in $\cT$. This quantity is conditional on both the observer and the target task family.
\end{definition}

These three quantities play different roles. A formal quantity is a statement about the idealized object of study. An operational proxy is a pipeline-dependent measurement. Task-relevant structure is a task-conditional notion. A single scalar need not faithfully represent all three.

The empirical analysis fixes a single pretraining-and-probing pipeline: a causal language model for pretraining, a compression-style scalar computed from pretraining validation loss, and a frozen-feature probe whose OOD probe accuracy is the reported task metric. Throughout the paper, \emph{fixed probe evaluation} is shorthand for this full pretraining-and-probing setup together with OOD probe accuracy as the task metric. Let $\OODPerf(\cD \to \cT)$ denote task-specific OOD performance as measured by OOD probe accuracy under that fixed probe evaluation, not a universal evaluation functional over all possible observers and tasks.

\begin{definition}[Monotone explanatory claim]
Fix a target task family $\cT$ and a fixed probe evaluation. A task-agnostic proxy $\cS$ supports the \emph{monotone explanatory claim} in this setting if, whenever two corpora $\cD_A$ and $\cD_B$ satisfy
\[
\cS(\cD_A) > \cS(\cD_B),
\]
task-specific OOD performance under that same evaluation does not put $\cD_B$ ahead of $\cD_A$:
\[
\OODPerf(\cD_A \to \cT) \ge \OODPerf(\cD_B \to \cT).
\]
Equivalently, if the proxy puts $\cD_A$ ahead of $\cD_B$, then OOD probe accuracy under that fixed probe evaluation should not end up higher for $\cD_B$ than for $\cD_A$.
\end{definition}

\subsection{A Controlled Counterexample}

The next proposition gives one such counterexample for that fixed pretraining-and-probing setup.

\subsection{Setup}

Let $B$ denote a latent \emph{background variable} and $R$ a latent \emph{relevance variable}. We generate observed sequences $X$ by combining two mechanisms:
\begin{itemize}[leftmargin=1.2em]
    \item a background generator $g_B(B)$ that contributes abundant but task-irrelevant structure;
    \item a relevance generator $g_R(R)$ that contributes less total structure but determines task labels.
\end{itemize}
The target task is to predict a label $Y = h(R)$ under a distribution shift that preserves the mechanism $h$ while perturbing surface-level aspects of $g_B$.

Let $\cS(\cD)$ be a scalar proxy for \emph{total learned structure} under this fixed pretraining and measurement pipeline. We do not assume that $\cS$ equals formal epiplexity; the argument only requires that $\cS$ increases when the learner can extract more compressible regularity from the data while remaining agnostic to the target task family.

\subsection{Main proposition}

\begin{proposition}[Within a fixed probe evaluation, proxy rankings can reverse relative to task-specific OOD performance]
\label{prop:main}
Fix an observer class, a representation, a pretraining pipeline, and an OOD evaluation protocol. Then there exist two pretraining distributions $\cD_A$ and $\cD_B$, a target task family $\cT$, and a task-agnostic scalar proxy $\cS$ computed from that fixed pipeline such that
\[
\cS(\cD_A) > \cS(\cD_B)
\]
while
\[
\OODPerf(\cD_A \to \cT) < \OODPerf(\cD_B \to \cT).
\]
Thus, even within a single fixed pretraining-and-probing setup, a total-structure proxy can rank $\cD_A$ above $\cD_B$ while $\cD_B$ still achieves higher OOD probe accuracy under that setup.
\end{proposition}

\begin{proof}[Proof sketch]
Construct $\cD_A$ so that $X$ contains a high-volume, highly compressible background component generated by $g_B$, together with a weak or rare relevance signal generated by $g_R$. Construct $\cD_B$ so that the background component is simpler or less frequent, but the relevance signal is stronger, more stable, or more directly recoverable. Under this construction, the fixed learner extracts more total regularity from $\cD_A$, so $\cS(\cD_A) > \cS(\cD_B)$. Yet the target task depends only on $R$, and the OOD shift perturbs the surface background statistics while preserving the $R \mapsto Y$ mechanism. Within that fixed probe evaluation, the representation learned from $\cD_B$ therefore supports stronger task-specific OOD performance on $\cT$ than the one learned from $\cD_A$. With sufficiently separated signal strengths, the two inequalities are strict.
\end{proof}

Appendix~A expands the same argument and makes the two strict inequalities explicit without changing the scope of the claim.

\begin{remark}
The proposition is an existential statement inside a fixed probe evaluation, not a universal impossibility theorem over every conceivable proxy, observer, or OOD evaluation metric. Positive empirical correlation can still hold under additional alignment assumptions, e.g., when high-volume structure and task-relevant structure are positively coupled.
\end{remark}

\subsection{A concrete sequence construction}

Proposition~\ref{prop:main} is instantiated below with token sequences.

Each training sequence has length $L = L_B + L_R$. The first $L_B$ positions are produced by a structured background process; the final $L_R$ positions carry a relevance signal.

\paragraph{Background process.}
Let the background subsequence be generated by a templated grammar with strong, nested regularity. For example:
\[
\texttt{OPEN}\ a_1\ a_2\ \cdots\ a_k\ \texttt{MID}\ a_1\ a_2\ \cdots\ a_k\ \texttt{CLOSE},
\]
possibly with periodic filler motifs and shallow Markov dependencies. This creates strong redundancy and long-range predictability.

\paragraph{Relevance process.}
Let the relevance subsequence contain a latent bit $R \in \{0,1\}$ encoded by a sparse motif relation, such as whether two rare marker tokens appear in matching parity classes or whether a mirror pair is present. The motif appears only on a fraction of examples; on the remaining examples the corresponding positions are occupied by distractor tokens that are independent of the label. The target label is $Y = R$.

\paragraph{Distribution $\cD_A$.}
In $\cD_A$, the background grammar is rich and stable while the relevance motif is weak, infrequent, or noisy.

\paragraph{Distribution $\cD_B$.}
In $\cD_B$, the background grammar is simpler while the relevance motif is frequent and stable.

\paragraph{OOD shift.}
At evaluation time we alter the surface realization of the background grammar --- e.g., swap symbol vocabularies, vary template lengths, or permute periodic fillers --- while preserving the motif-to-label mechanism for $R$.

Under this construction a learner can achieve lower pretraining loss and stronger compression-like statistics on $\cD_A$, yet derive features from $\cD_B$ that support stronger task-specific OOD performance on the target task.

\section{Experiments}
The experiments study one controlled empirical instantiation of the ordering-reversal pattern isolated by Proposition~\ref{prop:main}. They fix a single pretraining-and-probing setup: a four-layer causal language model for pretraining, a compression-style proxy derived from validation loss, and a frozen-feature probe whose OOD probe accuracy serves as the task metric. The goal is not to explore the full existential range of Proposition~\ref{prop:main}, but to test one concrete synthetic family designed to reproduce the same mechanism under fixed hyperparameters and seeds. The empirical question is whether this task-agnostic proxy can rank $\cD_A$ above $\cD_B$ even when $\cD_B$ yields higher OOD probe accuracy.

\subsection{Synthetic protocol}

Token sequences have length $64$ over a vocabulary of size $80$. Each sequence is the concatenation of a long background prefix and a six-token relevance suffix of the form
\[
[\texttt{marker}, u_1, \texttt{marker}, u_2, \texttt{query}, \texttt{answer}].
\]
The relevance suffix uses a dedicated 16-token relevance vocabulary split into two latent buckets. On informative examples, \texttt{answer} indicates whether $u_1$ and $u_2$ come from the same bucket; on non-informative examples, $u_1$ and $u_2$ are sampled from the full content vocabulary and the label is independent. We refer to the former as the \emph{informative subset}.

The background prefix is generated by a nested template process with \texttt{OPEN}/\texttt{MID}/\texttt{CLOSE} scaffolds, repeated chunks, and filler tokens. Three scalar controls govern this generator: background strength, repeat probability, and maximum nesting depth. The OOD split applies a harder background-shift regime, which permutes only background token identities and repetition patterns while preserving the relevance suffix semantics.

\paragraph{Primary evaluation configuration.}
Let
\[
\theta = (b, \rho, \eta, p, d)
\]
denote, respectively, background strength, relevance prevalence, relevance noise, repeat probability, and maximum nesting depth. The main configuration is asymmetric by construction:
\[
\theta_A = (0.95, 0.10, 0.08, 0.95, 5), \qquad
\theta_B = (0.30, 0.95, 0.01, 0.35, 2).
\]
Dataset A is therefore background-heavy and relevance-poor, whereas dataset B is background-light and relevance-rich.

\paragraph{Ablation configurations.}
The background-only ablation preserves the background coordinates $(b, p, d)$ from the primary evaluation configuration while setting the relevance coordinates to $(\rho, \eta) = (0.60, 0.03)$ for both datasets. The relevance-only ablation preserves the relevance coordinates $(\rho, \eta)$ from the primary evaluation configuration while setting the background coordinates to $(b, p, d) = (0.60, 0.65, 3)$ for both datasets. These are the only ablations reported in the paper.

\subsection{Models and training}

All reported results use the same fixed pretraining-and-probing setup and the same OOD probe-accuracy metric.
\begin{itemize}[leftmargin=1.2em]
    \item Model: 4-layer causal Transformer with width 128, 4 attention heads, MLP ratio 4, and dropout 0.1.
    \item Objective: next-token prediction.
    \item Optimizer: AdamW with learning rate $3\times 10^{-4}$ and weight decay 0.01.
    \item Data budget: 50k training sequences and 5k sequences each for validation, test, and OOD evaluation.
    \item Pretraining budget: 12 epochs for both datasets and all seeds.
\end{itemize}

After pretraining, the backbone is frozen and a one-hidden-layer MLP probe (hidden size 256, GELU, 100 epochs, learning rate $10^{-3}$, batch size 1024) is fit on the final hidden state of the prefix that excludes the answer token. The probe is trained on the labeled pretraining split and evaluated on held-out test and OOD splits. The OOD evaluation therefore remains inside a fixed probe-training budget, so the comparison is driven by what information the frozen representation exposes rather than by differences in labeled-data allocation across corpora.

\subsection{Metrics and success criteria}

The operational proxy is the compression-style score
\[
\cS(\cD) = -\min_t \mathcal{L}_{\mathrm{val}}^{(t)}(\cD) - \alpha \log(\max(2, N_{\mathrm{params}})),
\]
with $\alpha = 10^{-6}$. Because the architecture is fixed across datasets, higher $\cS$ primarily means lower best validation loss under the chosen observer.

Let $\Pi_{\mathrm{main}}(\cD)$ denote the protocol that fits the probe on the full labeled pretraining split and evaluates it on the full held-out test and OOD splits. Let $\Pi_{\mathrm{diag}}(\cD)$ denote the protocol that first restricts those three labeled splits to informative examples and then refits and evaluates the probe on that filtered data. We then report
\begin{align*}
M_{\mathrm{main}}(\cD) &= \operatorname{Acc}_{\mathrm{OOD}}(\Pi_{\mathrm{main}}(\cD)),\\
M_{\mathrm{diag}}(\cD) &= \operatorname{Acc}_{\mathrm{OOD}}(\Pi_{\mathrm{diag}}(\cD)).
\end{align*}
We refer to $\Pi_{\mathrm{diag}}$ as the \emph{filtered informative-subset diagnostic}. Because it changes the probe-training distribution as well as the evaluation subset, it is a secondary mechanism check rather than a pure test-time slice of $\Pi_{\mathrm{main}}$.

For each seed we summarize the run by three gaps:
\begin{align*}
\Delta_{\mathrm{proxy}} &= \cS(\cD_A) - \cS(\cD_B),\\
\Delta_{\mathrm{main}} &= M_{\mathrm{main}}(\cD_B) - M_{\mathrm{main}}(\cD_A),\\
\Delta_{\mathrm{diag}} &= M_{\mathrm{diag}}(\cD_B) - M_{\mathrm{diag}}(\cD_A).
\end{align*}
A seed realizes the \emph{all-sample reversal} when $\Delta_{\mathrm{proxy}} > 0$ and $\Delta_{\mathrm{main}} > 0$, and it satisfies the \emph{diagnostic criterion} when $\Delta_{\mathrm{proxy}} > 0$ and $\Delta_{\mathrm{diag}} > 0$. The fixed seed set $\{42, 123, 456\}$ is used throughout.

\subsection{Primary Empirical Criterion}

The primary empirical event of interest is the all-sample reversal in which the proxy ranks dataset A above dataset B while OOD probe accuracy is higher for dataset B:
\[
\Delta_{\mathrm{proxy}} > 0 \quad\text{and}\quad \Delta_{\mathrm{main}} > 0.
\]
Because $\Delta_{\mathrm{main}}$ is defined as dataset B minus dataset A, a positive value means that dataset B attains higher OOD probe accuracy even though the proxy ranks dataset A above B. A positive diagnostic gap is reported separately as mechanistic support and does not replace the primary all-sample reversal.

\subsection{Primary Evaluation Results}

Table~\ref{tab:main-protocol-results} summarizes the three-seed primary evaluation. In every seed, the proxy places dataset A ahead of dataset B. The all-sample reversal appears in two of the three seeds, and the filtered informative-subset diagnostic yields positive diagnostic gaps in all three seeds. Because $\Delta_{\mathrm{proxy}} > 0$ in every row of the primary evaluation, a positive diagnostic gap is equivalent here to satisfying the diagnostic criterion.
We therefore interpret the all-sample result as the primary empirical instantiation of the ordering reversal, with the filtered diagnostic and the ablations serving as corroborating mechanistic support rather than as replacement success criteria.

\begin{table}[t]
\centering
\caption{Three-seed summary for the primary evaluation. A positive ``gap B$-$A'' means dataset B achieves higher OOD probe accuracy than dataset A.}
\label{tab:main-protocol-results}
\begin{tabular}{lrrr}
\toprule
Seed & Proxy gap A$-$B & Main gap B$-$A & Diagnostic gap B$-$A \\
\midrule
42 & 1.2406 & -0.0062 & 0.0153 \\
123 & 1.2390 & 0.2494 & 0.2691 \\
456 & 1.2168 & 0.4742 & 0.4836 \\
\midrule
Mean & 1.2321 & 0.2391 & 0.2560 \\
\bottomrule
\end{tabular}
\end{table}

Figure~\ref{fig:main-protocol-gaps} visualizes the same pattern. The main OOD gap is slightly negative for one seed but positive for the other two, whereas the diagnostic gap remains positive throughout. Within this three-seed panel, the sign change occurs only in the all-sample OOD gap; the proxy gap stays positive in every seed, and the filtered informative-subset diagnostic gap stays positive throughout.

\begin{figure}[t]
\centering
\includegraphics[width=\linewidth]{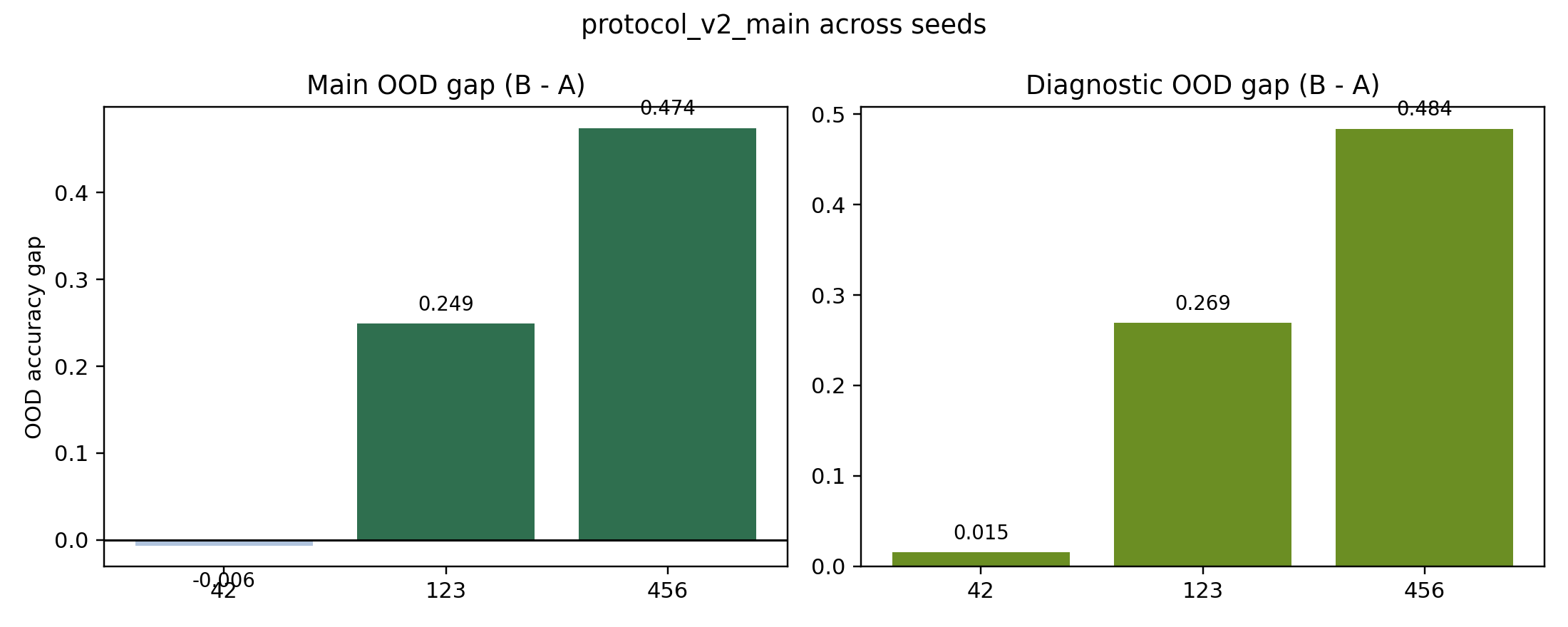}
\caption{Per-seed OOD probe-accuracy gaps for the primary evaluation. The left panel shows the all-sample gap and the right panel shows the filtered informative-subset diagnostic gap. Across these runs, the proxy gap remains positive, while the all-sample OOD probe gap varies in magnitude and sign.}
\label{fig:main-protocol-gaps}
\end{figure}

\subsection{Background-Only Ablation}

The three-seed background-only ablation tests the effect of removing the relevance contrast while keeping the background contrast. Here the proxy gap remains positive in all three seeds and is even slightly larger on average than in the primary evaluation (mean proxy gap $1.2675$ versus $1.2321$). By contrast, the OOD probe gaps almost vanish: the mean all-sample OOD gap drops to $+0.0057$ and the mean filtered informative-subset diagnostic gap becomes slightly negative at $-0.0025$. Only one of the three seeds realizes the all-sample reversal and two of the three satisfy the diagnostic criterion. The slightly negative diagnostic mean is not inconsistent with the latter count: two seeds show only tiny positive diagnostic gaps, while one seed contributes a somewhat larger negative gap. In this protocol, background structure can dominate a total-structure proxy without by itself producing a reliable OOD-performance advantage under the fixed probe evaluation.

Figure~\ref{fig:config-summary} compares the config-level means for the primary evaluation and the background-only ablation. The key pattern is the selective collapse of OOD probe performance once the relevance contrast is removed, despite the persistence of a large proxy gap.

\begin{figure}[t]
\centering
\includegraphics[width=\linewidth]{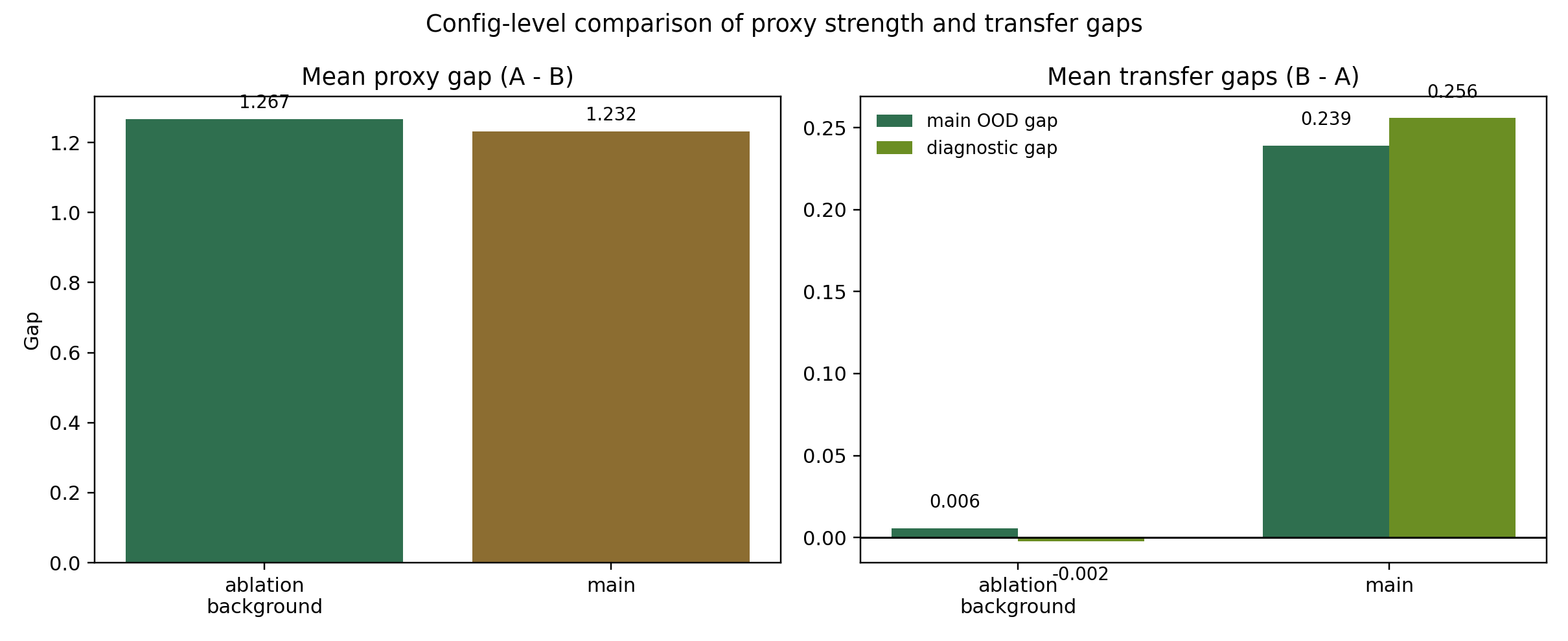}
\caption{Config-level comparison between the primary evaluation and the background-only ablation. The proxy gap remains large in both settings, but the OOD probe gaps collapse toward zero once the relevance contrast is removed. This isolates the role of task-relevant structure rather than total compressible structure.}
\label{fig:config-summary}
\end{figure}

\paragraph{Complementary relevance-only check.}
The relevance-only ablation tests the opposite intervention. Here the proxy gap is no longer positive in any seed: its mean becomes $-0.0389$, and neither the all-sample reversal nor the diagnostic criterion is satisfied in any seed. The mean all-sample and diagnostic OOD probe gaps remain only mildly positive at $+0.0164$ and $+0.0235$. Those positive means do not rescue the counterexample claim, because the reversal fails earlier: the proxy no longer ranks dataset A above dataset B in any seed. Together with the background-only ablation, these results are consistent with the view that the proxy gap in the primary evaluation is driven mainly by background regularities rather than by task-relevant structure alone.

\section{Limitations}
The scope of the paper is limited in several ways. First, the theory section gives an existential controlled counterexample within one fixed pretraining-and-probing setup and states the resulting proposition, but it does not provide a universal impossibility theorem over every proxy, observer, or OOD evaluation metric. Second, the empirical metric in this paper is OOD probe accuracy under that setup rather than end-to-end finetuning or a broader evaluation suite. Third, the filtered informative-subset diagnostic changes the probe-training distribution as well as the evaluation subset, so it functions as a mechanistic check rather than as a pure test-time slice of the primary evaluation. Fourth, the primary all-sample reversal appears in two of the three seeds rather than uniformly, so the empirical claim is a seed-sensitive controlled counterexample rather than a universal empirical regularity. Fifth, we have not yet shown comparable behavior across substantially different model families.

The construction is also synthetic. The synthetic setting isolates causal factors, but it limits how directly the conclusion extrapolates to real pretraining corpora. In particular, the paper does not argue that all positive empirical correlations between total-structure proxies and OOD performance are spurious; it argues only that a proxy ranking dataset A above dataset B does not by itself guarantee higher OOD probe accuracy for dataset A than for dataset B. A broader empirical case would require more seeds, substantially different model families, and at least one non-synthetic task family.

This work studies small synthetic models, releases no human data, and introduces no deployable high-risk system. Any broader implication is therefore methodological: distinctions between total structure and task-relevant structure may affect how researchers evaluate data value, and overstated claims that total-structure proxies should preserve OOD rankings may mislead dataset selection or benchmark interpretation if they are treated as sufficient without additional alignment assumptions.

\section{Conclusion}
\subsection{Scope of the Counterexample}
The counterexample is most directly relevant to claims that, within a fixed pretraining-and-probing setup, a scalar proxy for total learned structure should keep the same dataset ordering when task-specific OOD performance is measured. In other words, if the proxy ranks dataset A above dataset B, OOD probe accuracy in that setup should not be higher for dataset B than for dataset A. It does \emph{not} imply that structure-based metrics are never useful. Instead, using such metrics in that role requires additional assumptions, such as:
\begin{enumerate}[leftmargin=1.2em]
    \item the dominant learnable regularities are positively aligned with the target task family;
    \item the representation exposes task-relevant cues rather than burying them beneath irrelevant redundancy;
    \item the proxy is sufficiently stable across the model-and-optimization pipeline.
\end{enumerate}

Such assumptions are substantive rather than automatic. In this setup, keeping proxy-based A/B rankings aligned with OOD probe accuracy requires tracking task-relevant structure more closely than a total-structure proxy alone can guarantee.

\subsection{Observer, Representation, and Task-Relevant Structure}
Under the fixed setup analyzed here, the empirical utility of a pretraining distribution depends on relations among at least three ingredients:
\begin{enumerate}[leftmargin=1.2em]
    \item an \textbf{observer} (model class, optimization process, compute budget),
    \item a \textbf{representation} (tokenization, serialization, feature exposure), and
    \item a \textbf{task family} (what counts as successful task-specific OOD performance).
\end{enumerate}
The relevant question is therefore not only how much structure exists in the abstract, but how much \emph{reusable, task-relevant structure} is accessible to a particular learner under a particular representation.

This synthetic instantiation makes that distinction empirical as well as conceptual. Theoretically, the paper gives an existential controlled counterexample for this setup. Empirically, one concrete instantiation of the same separation appears in a seed-sensitive form. Under the primary all-sample evaluation, dataset B attains higher OOD probe accuracy than dataset A in two of the three seeds even though the proxy ranks dataset A above B. The filtered informative-subset diagnostic refits the probe on the informative subset. It yields positive diagnostic gaps in all three seeds, and the primary-evaluation proxy gap is also positive in each of those seeds. Because that secondary protocol changes the probe-training distribution, we treat it as mechanistic support rather than as a replacement for the primary all-sample reversal. The two complementary ablations indicate that, in this protocol, background regularity alone is sufficient for a positive proxy gap, whereas relevance-only structure alone does not recreate a positive proxy gap. These auxiliary results support the interpretation of the primary all-sample reversal; they do not replace it as the primary empirical criterion.

These results place a boundary on using total-structure proxies to explain OOD rankings in this setting. Under the fixed setup analyzed here, a scalar proxy for total learned structure can rank dataset A above dataset B even when task-specific OOD performance ranks dataset B above dataset A, because total structure and task-relevant structure can diverge. Computationally bounded notions of structure remain relevant, but this explanatory use requires narrower alignment conditions than a total-structure proxy alone can guarantee. Extending this boundary result into a broader empirical claim would require more seeds, more model families, and evidence beyond the present synthetic task family.

\bibliographystyle{plainnat}
\bibliography{references}

\appendix
\section{Additional Protocol Details}

\subsection{Label-query task family}
The label-query sequence family used here has relevance suffix
\[
[\texttt{marker}, u_1, \texttt{marker}, u_2, \texttt{query}, \texttt{answer}],
\]
where the relevance vocabulary contains 16 tokens split into two latent buckets. On informative examples, \texttt{answer} is \texttt{label1} iff $u_1$ and $u_2$ come from the same bucket, and \texttt{label0} otherwise. With probability equal to the relevance-noise parameter, we flip the bucket assignment used to generate $u_2$. On non-informative examples, $u_1$ and $u_2$ are sampled from the full content vocabulary and the label is sampled independently; these examples lie outside the informative subset.

\subsection{Background generator}
The background prefix is produced by a nested \texttt{OPEN}/\texttt{MID}/\texttt{CLOSE} template with repeated chunks and filler tokens. The background-strength parameter controls how often filler positions copy a template-driven pattern instead of random content. The repeat-probability parameter controls whether right-side chunks mirror left-side chunks, and the maximum nesting depth caps template depth. The harder background-shift regime permutes only background token identities and repetition patterns; the relevance suffix remains in the canonical vocabulary so that the label rule is preserved.

\subsection{Config summary}
Using the notation $\theta = (b, \rho, \eta, p, d)$ from Section~4, where the coordinates denote background strength, relevance prevalence, relevance noise, repeat probability, and maximum nesting depth, the primary evaluation uses
\[
\theta_A = (0.95, 0.10, 0.08, 0.95, 5), \qquad
\theta_B = (0.30, 0.95, 0.01, 0.35, 2),
\]
with dataset A background-heavy and relevance-poor relative to dataset B.

The background-only ablation keeps the background coordinates $(b, p, d)$ from the primary evaluation while setting both datasets to $(\rho, \eta) = (0.60, 0.03)$. The relevance-only ablation keeps the relevance coordinates $(\rho, \eta)$ from the primary evaluation while setting both datasets to $(b, p, d) = (0.60, 0.65, 3)$.

\subsection{Probe evaluation protocol}
All reported OOD probe-accuracy numbers use the same frozen backbone and the same one-hidden-layer MLP probe family. The reported runs use hidden size \texttt{256}, GELU activation, learning rate \texttt{1e-3}, \texttt{100} probe epochs, and batch size \texttt{1024}. The probe reads the final hidden state of the prefix that excludes the answer token.

The main OOD metric corresponds to $\Pi_{\mathrm{main}}$: it fits the probe on the full labeled pretraining split and evaluates it on the full held-out test and OOD splits. The filtered informative-subset diagnostic corresponds to $\Pi_{\mathrm{diag}}$: it first restricts all three labeled splits to informative examples and then repeats the same fitting-and-evaluation procedure on that filtered data. The resulting diagnostic therefore functions as a mechanistic check rather than as a pure test-time slice of the primary evaluation.

\subsection{Additional proof details for Proposition~\ref{prop:main}}
Proposition~\ref{prop:main} is existential and pipeline-relative. The argument does not require every conceivable proxy, observer, or OOD metric to fail. It only requires a fixed task-agnostic proxy whose value can be increased by compressible background regularity, together with a target task whose OOD success depends on a separate relevance variable.

More concretely, the construction uses three ingredients.
\begin{enumerate}[leftmargin=1.2em]
    \item The scalar proxy $\cS$ is computed from pretraining behavior under a fixed learner and therefore rewards regularities that reduce next-token prediction loss, whether or not those regularities are useful for the target task.
    \item The target label depends only on the relevance variable $R$, so background regularity can improve $\cS$ without directly determining the target task.
    \item The OOD shift perturbs the surface realization of the background process while preserving the mechanism $R \mapsto Y$.
\end{enumerate}

Under these conditions, choose $\cD_A$ so that the background generator is stronger, more repetitive, and easier for the fixed learner to model, while the relevance signal is weaker, rarer, or noisier. Choose $\cD_B$ so that the background generator is less redundant but the relevance signal is stronger and more stable.

\paragraph{Why the proxy can still rank dataset A above dataset B.}
Because both pretraining and the proxy are computed on full sequences, the proxy can be driven upward by large amounts of compressible background regularity. When the background advantage in $\cD_A$ is made sufficiently large relative to the relevance suffix, the fixed learner achieves lower validation loss on $\cD_A$ than on $\cD_B$, so the induced proxy gap becomes strictly positive: $\cS(\cD_A) > \cS(\cD_B)$.

\paragraph{Why OOD probe accuracy can still favor dataset B.}
The target task ignores the background variable and depends only on $R$. If $\cD_B$ presents the relevance pattern more often and with less noise, then its frozen representation exposes the task-relevant cue more reliably to the probe. Once the OOD shift disrupts background statistics while preserving the label rule, the extra background regularity in $\cD_A$ need not translate into stronger OOD probe accuracy. With sufficiently separated relevance prevalence and noise levels, the induced OOD ordering is also strict in the opposite direction: $\OODPerf(\cD_A \to \cT) < \OODPerf(\cD_B \to \cT)$.

\paragraph{What the proposition does and does not show.}
This expanded argument clarifies why the proposition should be read as a boundary result for the fixed probe evaluation used in this paper. It shows that a task-agnostic total-structure proxy can rank dataset A above dataset B even when dataset B still achieves higher OOD probe accuracy under that same setup. It does not show that all structure-based metrics fail, nor that positive proxy--performance correlations are impossible once additional alignment assumptions are imposed.

\section*{NeurIPS Paper Checklist}

\begin{enumerate}

\item {\bf Claims}
    \item[] Question: Do the main claims made in the abstract and introduction accurately reflect the paper's contributions and scope?
    \item[] Answer: \answerYes{}
    \item[] Justification: The abstract and Sections~1, 3, 4, 5, and 6 state a controlled counterexample claim with explicit scope conditions. They also note that the primary all-sample result appears in two of the three seeds and that the filtered informative-subset diagnostic is presented as mechanistic support rather than as the primary evaluation.

\item {\bf Limitations}
    \item[] Question: Does the paper discuss the limitations of the work performed by the authors?
    \item[] Answer: \answerYes{}
    \item[] Justification: Section~5 discusses the synthetic setting, the fixed pretraining-and-probing setup and its OOD probe-accuracy metric, the filtered informative-subset diagnostic, the seed sensitivity of the primary all-sample result (two of the three seeds), the fact that the theory section provides a proof sketch in the main text and additional proof details in the appendix rather than a fully formal proof, and the lack of evidence across substantially different model families.

\item {\bf Theory assumptions and proofs}
    \item[] Question: For each theoretical result, does the paper provide the full set of assumptions and a complete (and correct) proof?
    \item[] Answer: \answerNo{}
    \item[] Justification: Section~3 states the construction and assumptions for the main proposition, and Appendix~A adds explicit proof details for that proposition, but the submission still does not provide a fully formal proof.

    \item {\bf Experimental result reproducibility}
    \item[] Question: Does the paper fully disclose all the information needed to reproduce the main experimental results of the paper to the extent that it affects the main claims and/or conclusions of the paper (regardless of whether the code and data are provided or not)?
    \item[] Answer: \answerYes{}
    \item[] Justification: Sections~4.1--4.7 describe the data construction, model, optimizer, seeds, and evaluation logic, and the supplemental material includes anonymized code, protocol configurations, aggregate result files, and minimal reproduction commands.

\item {\bf Open access to data and code}
    \item[] Question: Does the paper provide open access to the data and code, with sufficient instructions to faithfully reproduce the main experimental results, as described in supplemental material?
    \item[] Answer: \answerYes{}
    \item[] Justification: The supplemental material contains anonymized code, configs, aggregate results, and a reviewer-facing README with reproduction commands. The data are synthetic and can be regenerated directly from the released scripts.

\item {\bf Experimental setting/details}
    \item[] Question: Does the paper specify all the training and test details (e.g., data splits, hyperparameters, how they were chosen, type of optimizer) necessary to understand the results?
    \item[] Answer: \answerYes{}
    \item[] Justification: Sections~4.1--4.7 specify the synthetic distributions, model architecture, optimizer, data splits, and seed-based evaluation protocol, while the supplemental YAML files provide the exact runnable settings.

\item {\bf Experiment statistical significance}
    \item[] Question: Does the paper report error bars suitably and correctly defined or other appropriate information about the statistical significance of the experiments?
    \item[] Answer: \answerNo{}
    \item[] Justification: We report per-seed outcomes and success counts across three seeds, but we do not yet provide error bars, confidence intervals, or formal significance tests. The evidence should therefore be read as multi-seed directional evidence rather than a complete statistical analysis.

\item {\bf Experiments compute resources}
    \item[] Question: For each experiment, does the paper provide sufficient information on the computer resources (type of compute workers, memory, time of execution) needed to reproduce the experiments?
    \item[] Answer: \answerNo{}
    \item[] Justification: The submission does not report execution time, peak memory, or a full compute budget breakdown for each run, so it does not yet support a complete ``Yes'' answer.
    
\item {\bf Code of ethics}
    \item[] Question: Does the research conducted in the paper conform, in every respect, with the NeurIPS Code of Ethics \url{https://neurips.cc/public/EthicsGuidelines}?
    \item[] Answer: \answerYes{}
    \item[] Justification: The work uses synthetic data only, involves no human subjects, and releases only anonymized code and aggregate results. We are not aware of any aspect of the project that conflicts with the NeurIPS Code of Ethics.

\item {\bf Broader impacts}
    \item[] Question: Does the paper discuss both potential positive societal impacts and negative societal impacts of the work performed?
    \item[] Answer: \answerYes{}
    \item[] Justification: Section~5 notes that the direct societal footprint is limited because the work is synthetic, but also discusses the indirect positive impact of sharper evaluation and the negative impact of overstated proxy-based explanations for data valuation.
    
\item {\bf Safeguards}
    \item[] Question: Does the paper describe safeguards that have been put in place for responsible release of data or models that have a high risk for misuse (e.g., pre-trained language models, image generators, or scraped datasets)?
    \item[] Answer: \answerNA{}
    \item[] Justification: The paper does not release a high-risk model, a scraped dataset, or any system intended for open deployment. The supplemental material contains only small synthetic-data code and aggregate experiment files.

\item {\bf Licenses for existing assets}
    \item[] Question: Are the creators or original owners of assets (e.g., code, data, models), used in the paper, properly credited and are the license and terms of use explicitly mentioned and properly respected?
    \item[] Answer: \answerYes{}
\item[] Justification: The supplementary material includes a dependency inventory for the released code path, together with pinned versions, reported licenses, and official project pages for the external software assets it uses.

\item {\bf New assets}
    \item[] Question: Are new assets introduced in the paper well documented and is the documentation provided alongside the assets?
    \item[] Answer: \answerYes{}
    \item[] Justification: The supplementary material documents the new synthetic generation and protocol code, includes runnable configurations, and provides a reviewer-facing README with minimal reproduction commands.

\item {\bf Crowdsourcing and research with human subjects}
    \item[] Question: For crowdsourcing experiments and research with human subjects, does the paper include the full text of instructions given to participants and screenshots, if applicable, as well as details about compensation (if any)? 
    \item[] Answer: \answerNA{}
    \item[] Justification: The paper does not involve crowdsourcing, surveys, or any other study with human participants.

\item {\bf Institutional review board (IRB) approvals or equivalent for research with human subjects}
    \item[] Question: Does the paper describe potential risks incurred by study participants, whether such risks were disclosed to the subjects, and whether Institutional Review Board (IRB) approvals (or an equivalent approval/review based on the requirements of your country or institution) were obtained?
    \item[] Answer: \answerNA{}
    \item[] Justification: The paper does not involve research with human subjects, so IRB approval is not applicable.

\item {\bf Declaration of LLM usage}
    \item[] Question: Does the paper describe the usage of LLMs if it is an important, original, or non-standard component of the core methods in this research? Note that if the LLM is used only for writing, editing, or formatting purposes and does \emph{not} impact the core methodology, scientific rigor, or originality of the research, declaration is not required.
    %this research? 
    \item[] Answer: \answerNA{}
    \item[] Justification: No LLM is used as an important or non-standard component of the proposed method, data generation, or evaluation pipeline. Any incidental writing assistance would therefore fall outside the declaration requirement stated in the question.

\end{enumerate}

\end{document}